\documentclass[12pt,review]{elsarticle}  

\usepackage{graphicx,xcolor,natbib,geometry}
\usepackage{latexsym}

\usepackage{amsmath,amssymb}
\usepackage{mathrsfs}
\renewcommand{\mathcal}{\mathscr}
\definecolor{DarkBlueUMA}{RGB}{0,45,93} 

\usepackage{enumerate}

\allowdisplaybreaks

\newtheorem{corollary}{Corollary}
\newtheorem{proposition}{Proposition}
\newtheorem{theorem}{Theorem}
\newtheorem{remark}{Remark}
\newtheorem{example}{Example}
\newcommand{\R}{\mathbb{R}} 

\def\konproof{\rm\hspace*{\fill}$\Box$}
\def\prooftxt{\mbox{\large\sc proof: }}
\newenvironment{proof}{\par\smallskip\noindent\prooftxt }%
{\konproof\par\vspace*{6pt}}

\begin{document}

\begin{frontmatter}

\title{A Critical Analysis of the Theoretical Framework\\ of the Extreme Learning Machine}

\author[A,A1]{Irina Perfilieva\corref{irina}}
\ead{irina.perfilieva@osu.cz}
\cortext[irina]{Corresponding Author}

\author[B]{Nicol\'as Madrid}
\ead{Nicolas.Madrid@uma.es}
\author[B]{Manuel Ojeda-Aciego}
\ead{aciego@uma.es}

\author[C]{Piotr Artiemjew}
\ead{artem@matman.uwm.edu.pl}
\author[C]{Agnieszka Niemczynowicz}
\ead{niemaga@matman.uwm.edu.pl}

\address[A]{University of Ostrava\\Institute for Research and Applications of Fuzzy Modeling\\30. dubna 22, 701 03 Ostrava 1, Czech Republic}
\address[A1]{Faculty of Computer Science and Telecommunications, Cracow University of Technology, Krakow, Poland}
\address[B]{Universidad de M\'alaga. Departamento de Matem\'atica Aplicada\\ Blv. Louis Pasteur 35. 29071 M\'alaga, Spain}
\address[C]{University of Warmia and Mazury in Olsztyn\\ Faculty of Mathematics and Computer Science, Olsztyn, Poland}

\begin{abstract}

Despite the number of successful applications of the Extreme Learning Machine (ELM), we show that its underlying  foundational   principles do not have a rigorous mathematical justification. Specifically, we refute the proofs of two main statements, and we also create a dataset that provides a counterexample to the ELM learning algorithm and explain its design, which leads to many such counterexamples. Finally, we provide alternative statements of the foundations, which justify the efficiency of ELM in some theoretical cases. 
\end{abstract}

\end{frontmatter}

\section{\label{S1}Introduction}

The Extreme Learning Machine provides a machine learning algorithm known for its simplicity and efficiency. It was proposed by Huang,  Zhu, and  Siew in \cite{Huang06} as a fast and effective learning method for Single-hidden Layer Feed-forward Neural networks (SLFN).
The key idea behind ELM is that 
``\textit{hidden neurons need not be tuned with the consideration of neural network generalization theory, control theory, matrix theory and linear system theory}''~\cite{Huang15}. 

Essentially, ELM consists in  randomly initializing the weights connecting the input layer to the hidden layer and, then, analytically determine the weights connecting the hidden layer to the output layer. This analytical solution makes the learning process much faster compared to traditional gradient-based methods. According to \cite{SusCam20}, training a network using an ELM is computationally inexpensive compared to evolutionary optimization and classical neural network training algorithms.

Due to the simplicity of the proposed model and the proclaimed principle of universal approximation, this methodology has gained popularity among many researchers. ELM theory and methodology have been widely used from both a theoretical and an applied point of view. Within the former, we can find 
contributions related to the construction of regression models \cite{Liu2023,Zhang202310589,Zhao2023,Wang202312367}, classification \cite{Cai202121,Wu202310951,Zhang2023}, or recognition \cite{Ouyang2021,Peng2022596,Schiassi:2021aa}; concerning applications, we can find  them among others in the fields of medicine \cite{Chakravarthy202347585,Jiang20233631,Kuila202329857}, weather forecast \cite{Jiang2023,Zhang202311059}, financial modeling \cite{Novykov20231581,Zhu202050}, or facial emotion or sign language recognition \cite{Banskota20236479,Sonugur20238553}.

The scientific community, perhaps in view of the numerous successful applications, accepts on faith the theoretical framework proposed in Huang’s pioneering work \cite{Huang03,Huang06}  and it has been never revised. 
Moreover, to demonstrate success, many methods using ELM limit datasets, ignore the randomness of input weights, or neglect the contribution of bias to confidence interval estimates. Only a few papers present estimates to obtain forecast confidence intervals or forecasts using ELM, which also provides an exhaustive characterization of ELM as an SLFN with random input weights and biases, allowing optimization of output weights using a least squares procedure. 
 
The purpose of this article is to show that the theoretical framework  in \cite{Huang06} is not supported by the rationale proposed there. To do this, we show the theoretical inconsistency of the proofs of the two main theorems given in \cite{Huang06}, and also provide some counterexamples that refute the theoretical foundations of the ELM. On the other hand, we propose several steps to build a strong mathematical foundation by placing stricter restrictions on the main theorems in \cite{Huang06}.

In order to do this, we carefully analyzed  proofs in the paper \cite{Huang06}, which many authors take for granted, and also references in the existing literature.
We found  articles, see \cite{Wang08} and references therein, where the authors  analyzed the proposed principles of the ELM algorithm  and came to the conclusion that they are similar to  feedforward networks, both Radial Basis Functions (RBF) and Multi-Layer Perceptrons (MLP), with randomly fixed hidden neurons, which have previously been proposed and discussed by other authors in papers and textbooks. Our research and analysis are much deeper than that of \cite{Wang08}, as we  do not focus on the philosophy but  on the methodology and on the technical details of the proofs in \cite{Huang06}.

We also take into account the main technical trick - randomness, which the authors of \cite{Huang06} put forward in a response to \cite{Wang08} as the most significant in their approach.
Our theoretical analysis is supported by the specially designed counterexample which refutes the statement: \textit{``SLFNs (with $N$ hidden nodes) with randomly chosen input weights and hidden layer biases ... can exactly learn $N$ distinct observations.''} This counterexample proposes a dataset of I/O pairs that ELM cannot reproduce, even if it uses multiple runs with randomly selected weights and biases.
We believe that with this counterexample, we have contributed to the topic of neural network interpretability (AIX) since the counterexample we created shows the limitations of ANN approximation capabilities.

The structure of the paper is as follows. After recalling in Section~\ref{prelim} the basic notions and notation of the ELM theoretical framework, we present in Section~\ref{S5} one counterexample to the learning algorithm proposed in \cite{Huang06}. The construction of this counterexample, which allows one to construct many similar ones, is explained in detail. Then, in section~\ref{S2}, we analyze the statements and proofs of \cite[Theorems 2.1 and 2.2]{Huang06} and show their shortcomings, both theoretically and with some pathological counterexamples.
In Section~\ref{S4} we comment on the experimental analysis done in  \cite{Huang06} that discusses the test cases computed by the proposed ELM learning algorithm and compares their performance characteristics with other learning strategies. We show that the actual processing of the proposed test case described in \cite{Huang06} differs from the ELM algorithm. In Section~\ref{SecTeoNew} we provide a new theoretical result in the line of \cite[Theorem~2.1]{Huang06} that supports the technique under certain circumstances. Finally, in Section~\ref{S7}, we show that the counterexample to the ELM learning algorithm proposed in Section~\ref{S5} can be learned by a corresponding ELM with different hyperparameters than those, which were stated in 
Theorem~2.1 from \cite{Huang06}. In Section~\ref{conclusion} we present some conclusions and prospects for future work.

\section{Preliminary definitions on the ELM theoretical framework}\label{prelim}

We use the notation introduced in \cite{Huang06} and assume that an SLFN is characterized by $\tilde N$ hidden nodes, an activation function $g\colon \R\to \R$, $n$-dimensional vector inputs $\mathbf{x}_i,\,i=1,\ldots, N$, and $m$-dimensional vector outputs $\mathbf{t}_i,\,i=1,\ldots, N$. All neurons in the hidden layer are fully connected to the inputs via the weights $w_{ij}, \,i=1,\ldots, \tilde N,\, j=1,\ldots, n$, and the bias vector $(b_1,\ldots, b_{\tilde N})$. Finally, the hidden layer neurons are fully connected to the outputs via the weight vectors $\beta_i=(\beta_{i1},\ldots, \beta_{im})^T, \,i=1,\ldots, \tilde N$. 
Graphically, the underlying architecture is characterized by a single hidden layer as depicted in Figure~\ref{figSLFN}.

\begin{figure}
$$
\includegraphics[scale=0.35]{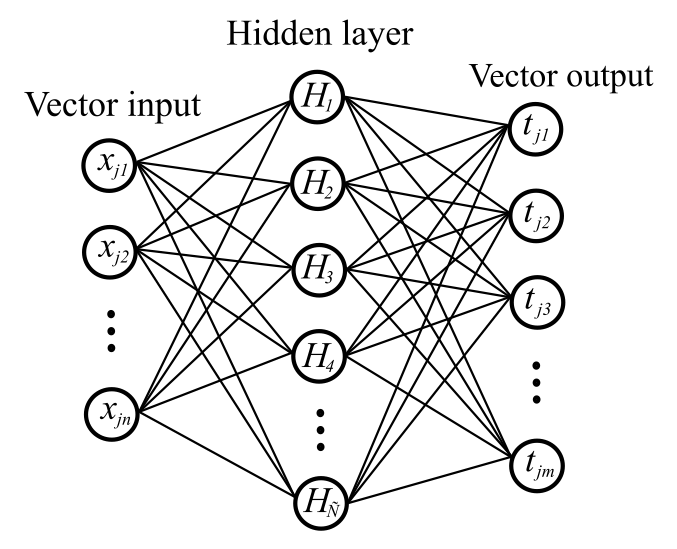}
$$
\vspace{-1cm}
\caption{Graphical representation of the structure of a Single-hidden Layer Feed-forward Neural network.}\label{figSLFN}
\end{figure}

SLFN training is based on $N$ different training samples $\{(\mathbf{x}_i,\mathbf{t}_i)\mid i=1,\ldots, N\}$, where $\mathbf{x}_i, \,i=1,\ldots, N$ are $n$-dimensional vector inputs, and $\mathbf{t}_i,\,i=1,\ldots, N$ are $m$-dimensional vector outputs. The computational model is given by
\begin{equation}
\label{SLFNmodel1}
\sum_{i=1}^{\tilde N}\beta_i g(\mathbf{w}_i\cdot \mathbf{x}_j+b_i)=\mathbf{o}_j,
\end{equation}
where $\mathbf{o}_j$ is the computed output vector corresponding to the input $\mathbf{x}_j,\, j=1,\ldots, N$.

Following~\cite{Huang06} again, we denote $h_{ji}=g(\mathbf{w}_i\cdot \mathbf{x}_j+b_i)$ and define the $N\times \tilde N$ real matrix  $\mathbf H =(h_{ji})$, with  $j=1,\ldots, N, i=1,\ldots, \tilde N$, which is called a \emph{hidden layer output matrix} of the neural network. Using $\mathbf H$, Equation~\eqref{SLFNmodel1} can be rewritten as the following matrix equation
\begin{equation}
\label{SLFNmodel2}
\mathbf H\beta=\mathbf T,
\end{equation}
where $\beta=(\beta_1^T,\ldots, \beta_m^T)$ is a $\tilde N\times m$ matrix formed by the weight vectors, and $\mathbf T$ is an $N\times m$ matrix of output vectors.

\begin{remark}
Thus, the equality (\ref{SLFNmodel2}) can be considered as a system of $N$ linear equations, in which the coefficients have the form $h_{ji}=g(\mathbf{w}_i\cdot \mathbf{x}_j+ b_i)$, $j=1,\ldots, N$, $i=1,\ldots, \tilde N$, where $N$ is the number of samples, $\tilde N$ is the number of neurons. For any given $\mathbf{w}_i$ and $b_i$, the system has $N\times \tilde N$ unknowns $\beta_{ik}$, where $k=1,\ldots, m$ is a reference to the output dimension. Therefore, for each fixed $k$, (\ref{SLFNmodel2}) can be rewritten as follows:
\begin{equation}
\label{SLFNmodel3}
\sum_{i=1}^{\tilde N}h_{ji}\cdot \beta_{ik}=\sum_{i=1}^{\tilde N} g(\mathbf{w}_i\cdot \mathbf{x}_j+b_i)\cdot\beta_{ik}=t_{jk}.
\end{equation}
Thus, for each fixed $k=1,\ldots,m$, the system (\ref{SLFNmodel2}) contains $N$ linear equations for $\tilde N$ unknowns. In other words, (\ref{SLFNmodel2}) is a set of $m$ independent linear subsystems of $N$ linear equations with $\tilde N$ unknowns such that all subsystems have the same coefficient matrix $H=(h_{ji})$.
To solve the system (\ref{SLFNmodel3}) for each fixed $k=1,\ldots,m$, the authors of the ELM assume that all values of input-output pairs $(\mathbf{x}_j,t_{jk})$, $j=1,\ldots, N$, are processed simultaneously.

For example, if the input-output pairs (samples) belong to a function of one variable, then to solve (\ref{SLFNmodel3}), it is necessary to have all samples simultaneously.

Below we reproduce two theorems in which the authors of ELM \cite{Huang06} consider the question of solvability of (\ref{SLFNmodel2}), respectively (\ref{SLFNmodel3}). The first theorem assumes that the matrix $H$ is square, i.e. $N=\tilde N$, and the second theorem assumes that $\tilde N< N$.
\end{remark}

\textbf{Theorem 2.1} \cite{Huang06}
{\it Given a standard SLFN with $N$ hidden nodes
and activation function $g\colon \R\to \R$ which is infinitely
differentiable in any interval, for $N$ arbitrary distinct samples $(x_i,t_i)$,where $x_i\in\R^n$ and $t_i\in\R^m$, for any $\mathbf{w}_i$ and $b_i$ randomly chosen from any intervals of $\R^n$ and $\R$, respectively, according to any continuous probability distribution, then with probability one, the hidden layer output matrix $\mathbf H$ of the SLFN is invertible, and $\|\mathbf H\beta-\mathbf T\|=0$}.\par\smallskip

 \textbf{Theorem 2.2} \cite{Huang06}   {\it Given any small positive value $\varepsilon >0$  and activation function $g\colon\R\to \R$ which is infinitely differentiable in any interval, there exists $\tilde N\leq N$ such that for $N$ arbitrary distinct samples $\{(\mathbf{x}_i,\mathbf{t}_i)\mid i=1,\ldots, N\}$, where $\mathbf{x}_i\in \R^n$, and $\mathbf{t}_i\in \R^m$, for any $\mathbf{w}_i$ and $b_i$ randomly chosen from any intervals of $\R^n$ and $\R$, respectively, according to any continuous probability distribution, then with probability one, $\|\mathbf H\beta-\mathbf T\|<\varepsilon$}.\par\smallskip

\noindent\textbf{Algorithm ELM} \cite[section~3.3]{Huang06} 

 Given a training set $\aleph = \{(\mathbf x_i,\mathbf t_i) \mid\mathbf x_i\in \mathbb R^n,\mathbf t_i \in \mathbb R^m, i =1, \dots, N\}$, activation function $g(x)$, and hidden node number $\tilde N$,
\begin{description}
 \item[Step 1:] Randomly assign input weight $\mathbf w_i$ and bias $\mathbf b_i, i=1, \dots, \tilde N$.
 \item[Step 2:] Calculate the hidden layer output matrix $\mathbf H$. 
  \item[Step 3:] Calculate the output weight $\beta=\mathbf H^\dagger \mathbf T$,   where $\mathbf H^\dagger$ is the Moore-Penrose generalized inverse of matrix $\mathbf H$ and $\mathbf T =[\mathbf t_1, \dots, \mathbf t_n]^T$.
\end{description}

\section{Counterexample to the ELM Learning Algorithm}
\label{S5}

This section shows one counterexample to the ELM learning algorithm   which refutes the following statement: 
{\em
``SLFNs (with $N$ hidden nodes) with randomly chosen input weights and hidden layer biases \dots\ can exactly learn $N$ distinct observations.'' } This statement is taken from \cite{Huang06}, citing \cite{Huang03,TamTat97}.

Along with the description of the counterexample, its construction is explained, which makes it possible to construct many similar counterexamples. We will see that for the set of input-output pairs we created, the ELM learning algorithm does not return output values for the corresponding inputs used in training, even if the number of hidden nodes and different training samples are the same. Last but not least, ELM does not return the above mentioned output values even if it uses multiple runs with randomly selected weights and biases.

\begin{example}
\label{Ex3} 
The dataset $S\subset \R\times\R$ of $N=400$ different training samples is the union of the four sets, i.e. $S=\bigcup_{i=1}^4 S_i$, where
\begin{itemize}
\item $S_1=\{ (\frac{k\pi}{100}, e^{\frac{k\pi}{100}}\sin(\frac{k\pi}{10})) \mid k=1,\ldots, 100 \}$;
\item $S_2=\{ (\frac{k\pi}{100}+\frac{\pi}{400}, e^{-\frac{k\pi}{100}-\frac{\pi}{400}}\sin(\frac{k\pi}{10}+\frac{\pi}{40})) \mid k=1,\ldots, 100 \}$;
\item $S_3=\{ (\frac{k\pi}{100}+\frac{\pi}{200}, e^{\frac{k\pi}{100}+\frac{\pi}{200}}(-\sin(\frac{k\pi}{10}+\frac{\pi}{20}))) \mid k=1,\ldots, 100 \}$;
\item $S_4=\{ (\frac{k\pi}{100}+\frac{\pi}{300}, e^{-\frac{k\pi}{100}-\frac{\pi}{300}}(-\sin(\frac{k\pi}{10}+\frac{\pi}{30}))) \mid k=1,\ldots, 100 \}$;

\end{itemize}
The illustration of the dataset $S$ is given in Figure~\ref{DataSet}.

This set was used as the training set for the ELM learning algorithm, together with the sigmoidal activation function $g$ and hidden node numbers $\tilde N= 400,300,200,100,50$. In the first case, when the number of hidden nodes $\tilde N= 400$ and the number of training samples are the same, the calculated output values are far from those used for training. In all other cases, $\tilde N= 300,200,100,50$, the calculated output values are close to those calculated for $\tilde N= 400$ and far from those used for training. 

This contradicts  \cite[Remark 3]{Huang06}, which comments on the ELM algorithm by stating that ``for any infinitely differentiable activation function SLFNs with $N$ hidden nodes can learn $N$ distinct samples exactly and SLFNs may require less than $N$ hidden nodes if learning error is allowed''. The outputs of the ELM learning algorithm, trained separately for $\tilde N= 400$ and $\tilde N= 50$ hidden nodes, and applied to each input value of the dataset $S$, are shown in Figure~\ref{ELM400} together with training samples. It is clear that these results contradict both statements of Remark 3.
\begin{figure}
\includegraphics[height=5cm,width=8cm]{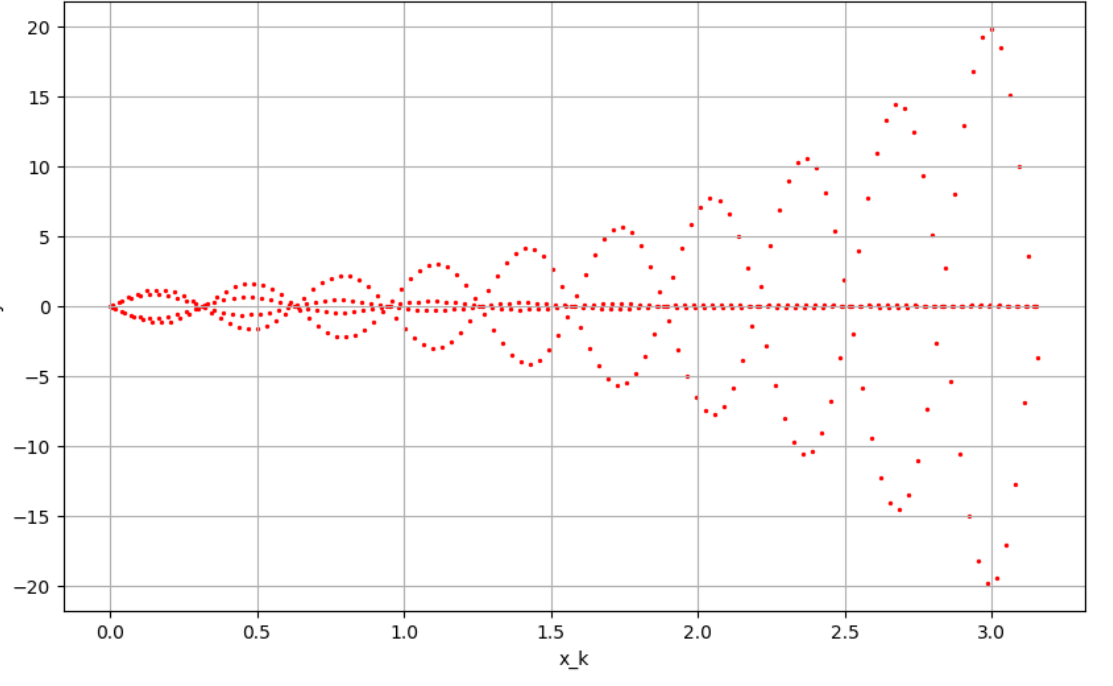}
\includegraphics[height=5.75cm,width=9cm]{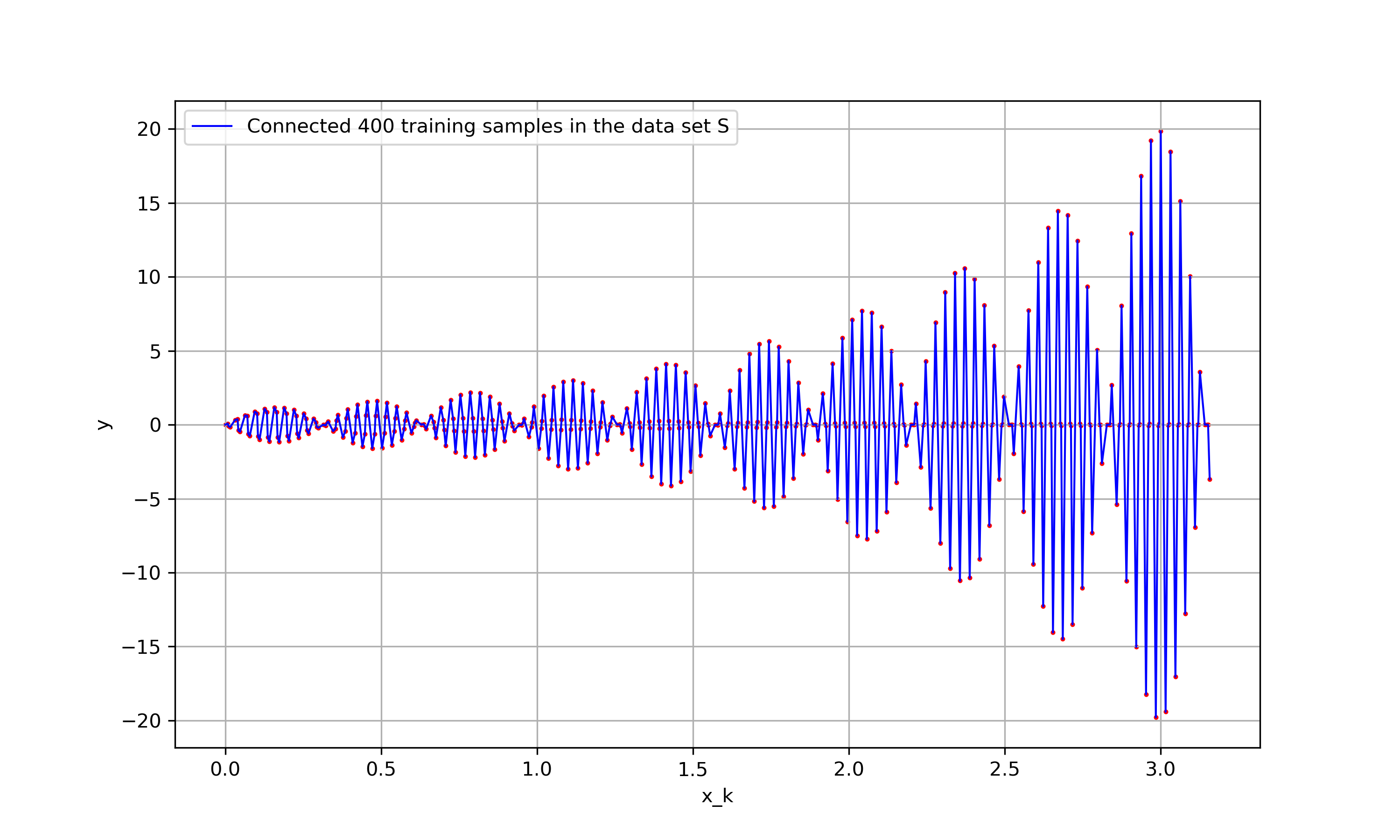}
\caption{\label{DataSet}Data set $S$ of $N=400$ different training samples for the ELM learning algorithm in two versions: pointwise (left) and piecewise connected (right).}
\end{figure}
\begin{figure}
\includegraphics[height=5.75cm,width=9cm]{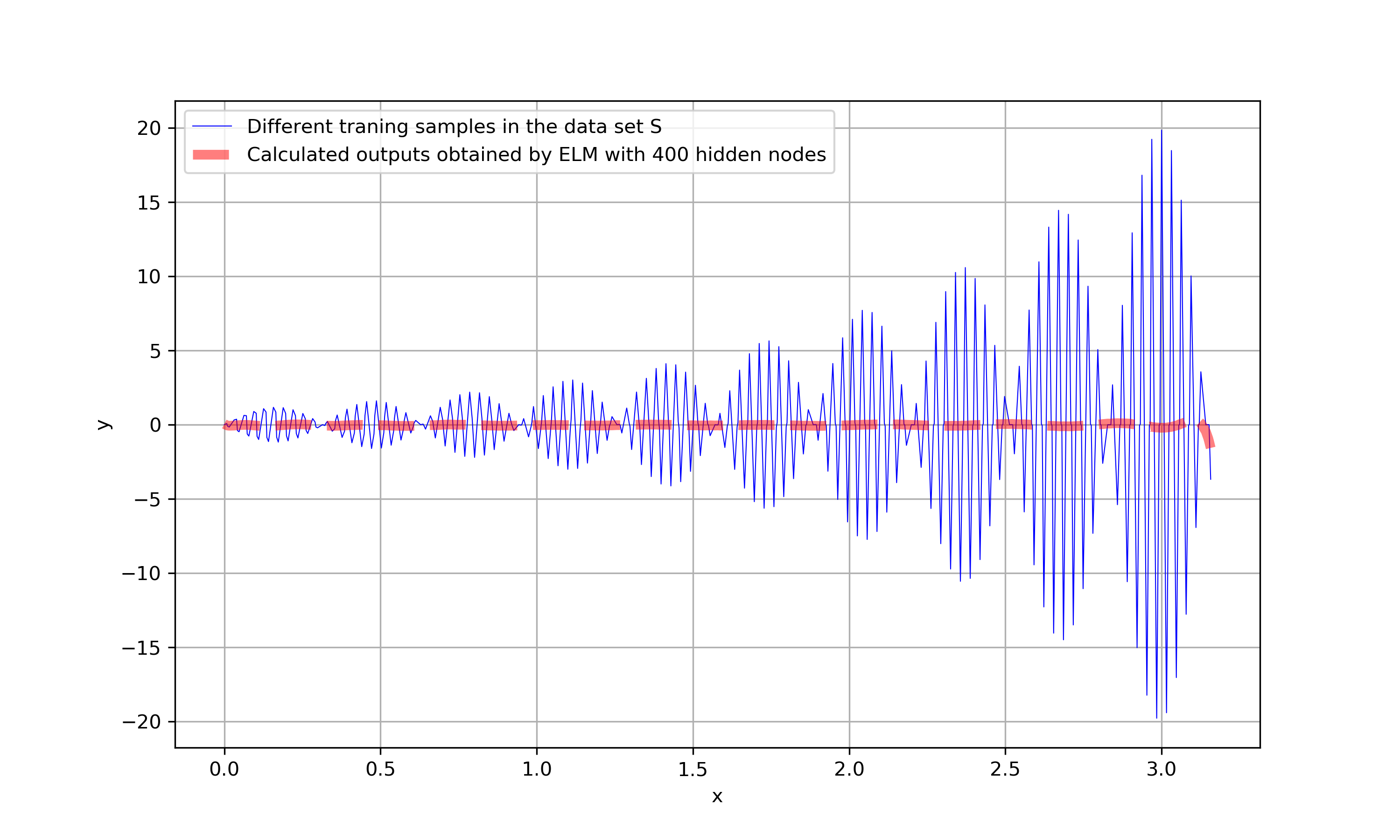}
\includegraphics[height=5cm,width=7.6cm]{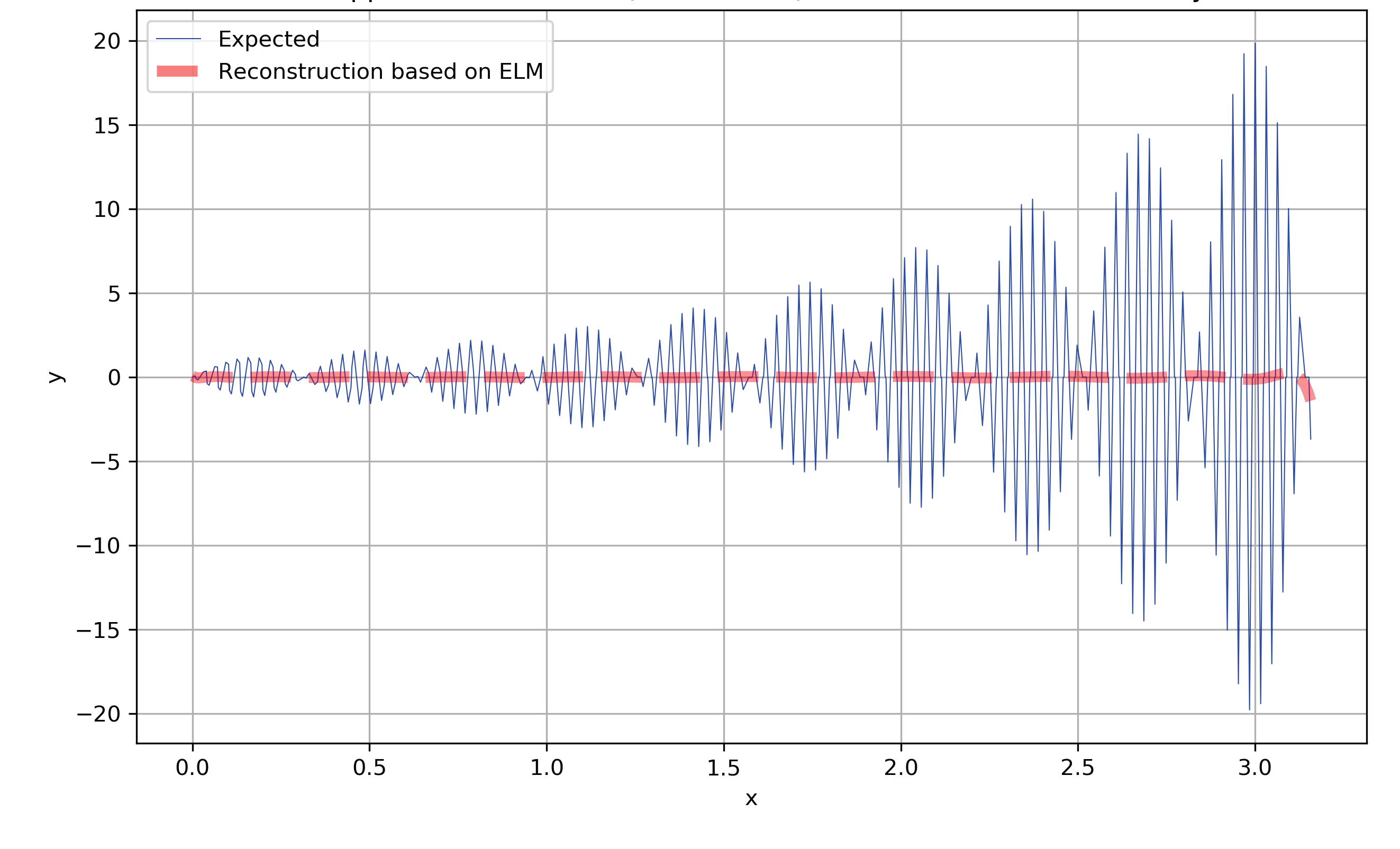}
\caption{\label{ELM400}Raw data from the set $S$ used for training (blue), and calculated output values (red) resulting from multiple runs with 
the ELM learning algorithm trained for $\tilde N= 400$ (left) and $\tilde N= 50$ (right) hidden nodes. Visually, both (left and right) calculated outputs (red) are far from those (blue) used for training.}
\end{figure}
\end{example}

Commenting on the results shown in Example~\ref{Ex3}, we see that the dataset $S$ is a counterexample to both  statements of Theorem 2.1 and Theorem 2.2. This means that if  input-output samples are taken from $S$, we will not reach the conclusions of the abovementioned theorems even if it uses multiple runs with randomly selected weights and biases.

In creating the dataset $S$, we combined two equally distributed datasets $S_1\cup S_3$ and $S_2\cup S_4$ with samples of sharp and soft oscillation (w.r.t. the zero value) functions. Since the ELM learning algorithm aims to minimize the mean squared error, it calculates output values close to zero. This conclusion is illustrated in Figure~\ref{ELM-400/50}, where we show the outputs of the ELM learning algorithm trained on $S$ for $\tilde N= 400$ (left) and $\tilde N= 50$ (right) hidden nodes.
\begin{figure}
\includegraphics[height=5cm,width=7.6cm]{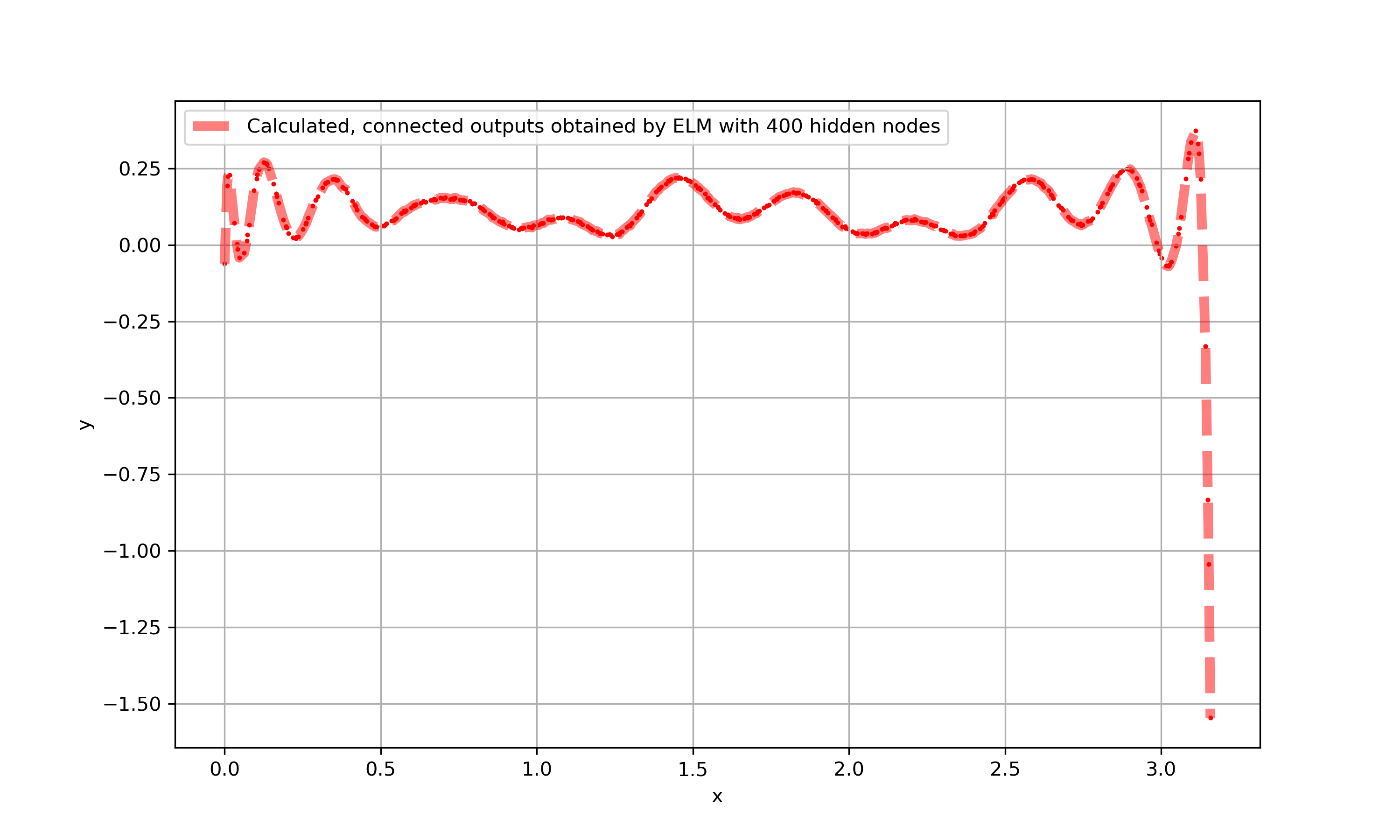}
\includegraphics[height=5cm,width=7.6cm]{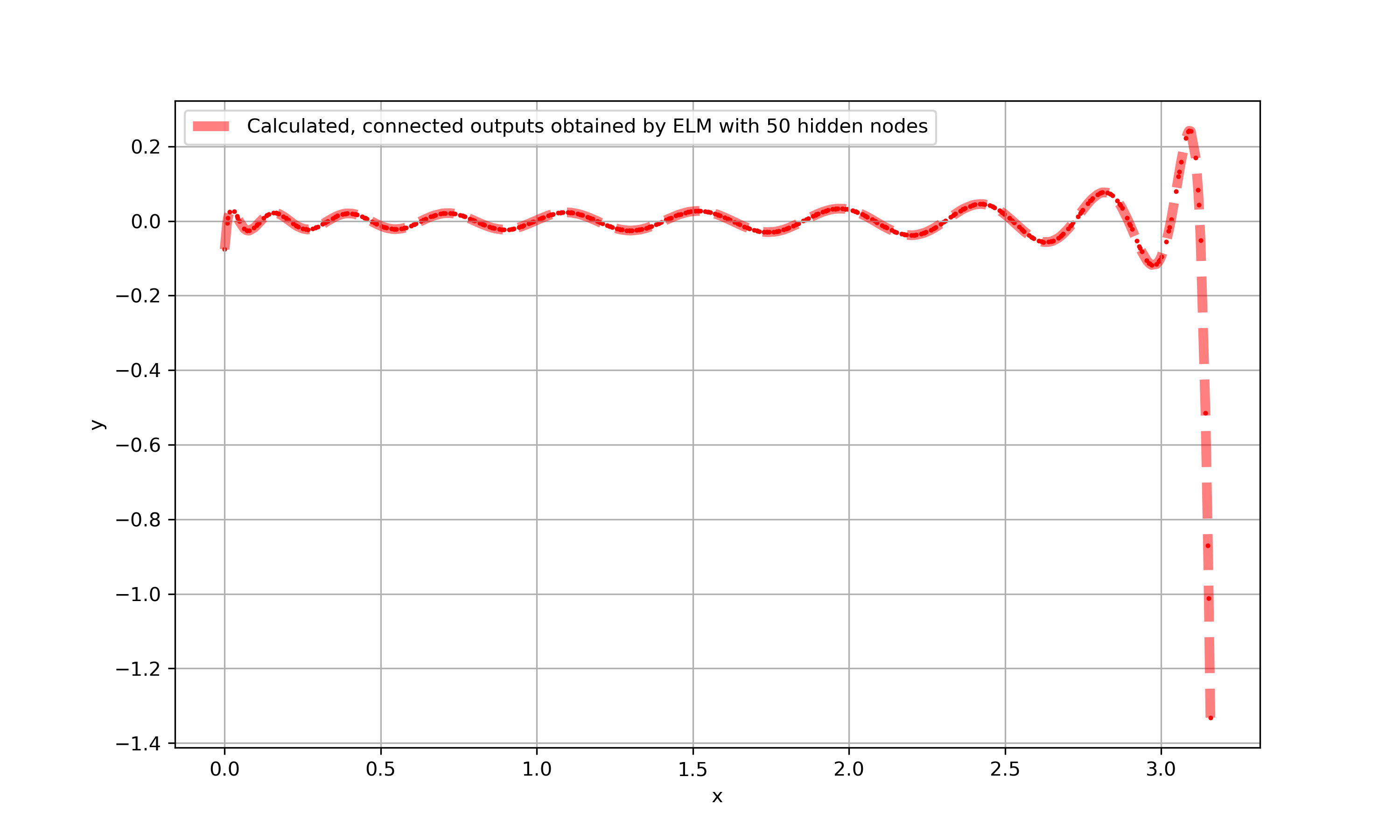}
\caption{Two graphs illustrate the results of several runs with randomly selected weights and biases of the ELM learning algorithm trained on hidden nodes of $\tilde N= 400$ (left) and $\tilde N= 50$ (right) on Example~\ref{Ex3} with dataset $S$ consisting of $N=400$ training pairs. Visually calculated results are close to zero.}\label{ELM-400/50}
\end{figure}

\section{The random choice of weights and bias values makes the ELM theoretical framework unsupported}
\label{S2}

In this section, we analyze the statements and proofs of two main theorems from \cite{Huang06}, which the author refers to 
as  
``\emph{rigorously proving that  for any infinitely differentiable activation function SLFNs with $N$ hidden nodes can learn $N$ distinct samples exactly and SLFNs may require less than $N$ hidden nodes if learning error is allowed.}''

\subsection{Refutation of the proof of  \cite[Theorem 2.1]{Huang06}}

For ease of reading let us repeat the statement of \cite[Theorem 2.1]{Huang06}.

\noindent\textbf{Theorem 2.1} \cite{Huang06}
{\it Given a standard SLFN with $N$ hidden nodes
and activation function $g\colon \R\to \R$ which is infinitely
differentiable in any interval, for $N$ arbitrary distinct samples $(x_i,t_i)$,where $x_i\in\R^n$ and $t_i\in\R^m$, for any $\mathbf{w}_i$ and $b_i$ randomly chosen from any intervals of $\R^n$ and $\R$, respectively, according to any continuous probability distribution, then with probability one, the hidden layer output matrix $\mathbf H$ of the SLFN is invertible, and $\|\mathbf H\beta-\mathbf T\|=0$}.\smallskip

The formulation above is confusing, in that the matrix norm is undefined, but the main problem with it is that its proof is not correct, as we will show below.

In the proof, the authors refer to the method of Tamura and Tateishi \cite{TamTat97} and their own previous work \cite{Huang03}. Our refutation is based on the following list of wrong claims appeared in the proposed proof of Theorem 2.1. In each item, we first cite the authors' arguments and then we show the mistake.

\begin{enumerate}
\item In the third paragraph of the proposed proof of Theorem 2.1 in~\cite{Huang06}, the authors state:

\begin{quotation}
\noindent {\it ``Since $\mathbf{w}_i$ are randomly generated based on a continuous
probability distribution, we can assume that 
$\mathbf{w}_i\cdot \mathbf{x}_k\neq \mathbf{w}_i \cdot \mathbf{x}_{k'}$} for all $k\neq k'$.'' 
\end{quotation}

To invalidate this assumption, we firstly  prove the following result.

\begin{proposition}\label{prop1}
Given two linearly independent vectors $\mathbf{x}_1,\mathbf{x}_2\in \R^n$, $n\geq 2$, and $v\in \R\smallsetminus \{0\}$, there is a non-zero $n$-dimensional vector $\mathbf{w}$ such that 
\begin{equation}
\label{eq1}
\mathbf{w}\cdot \mathbf{x}_{1}=\mathbf{w}\cdot \mathbf{x}_{2}=v.
\end{equation} 
\end{proposition} 

\begin{proof}
The proof of this statement follows directly from the Fredholm alternative. However, below, we present a constructive proof.
Let us assume, without loss of generality that the first two coordinates of $\mathbf{x}_1$ and $\mathbf{x}_2$ form linearly independent vectors. That is, $(x_{11},x_{12})$ is linearly independent to $(x_{21},x_{22})$. Let $X$ be the $2\times 2$ matrix with those vectors as rows. Then, the matrix equation
$$
X\cdot (w_1,w_2)^T=(v,v)
$$
has a unique non-zero solution. Let us name $(w^*_1,w^*_2)$ to such a solution. Then, the vector $\mathbf{w}=(w^*_1,w^*_2,0,\dots,0)$ solves Equation \eqref{eq1}.
\end{proof}

\begin{remark}
\label{rm1}
From Proposition~\ref{prop1}, it follows that 
\begin{itemize}
\item Under its condition we have infinitely many weight vectors $\mathbf{w}$ that solve the equation $\mathbf{w}\cdot \mathbf{x}_{1}=\mathbf{w}\cdot \mathbf{x }_{2}$ and thereby violate the condition put forward in the proposed proof of Theorem 2.1;
\item Under the reverse condition that among $N$ different training samples there are no two linearly independent vectors $\mathbf{x}_1, \mathbf{x}_2\in \R^n$, $n\geq 2$, we get the trivial conclusion that all vectors $\mathbf{x}_1,\ldots, \mathbf{x}_N$ are scaled versions of, say, $\mathbf{x}_1$, therefore \eqref{eq1} has infinitely many solutions.
\end{itemize}
\end{remark}

\item The second incorrect statement in the proof of~\cite{Huang06} is related to the vector $\mathbf{c}$  defined by
$$
\mathbf{c}(b_i)=(g(\mathbf{w}_i \cdot\mathbf{x}_1+b_i),\ldots, g(\mathbf{w}_i \cdot\mathbf{x}_N+b_i))^T.
$$
for $b_i\in (a,b)	\subset \mathbb{R}$. 
Specifically, in the proof one can read
\begin{quotation}
\noindent {\it ``Vector $\mathbf{c}$ does not belong to any subspace whose dimension is less than N.''} 
\end{quotation}
 
This formulation is clearly  wrong from the beginning, since every vector belongs to the subspace of dimension one generated by itself. Analyzing  the proof, it seems that the authors intended to show that the $N$-dimensional column vector $\mathbf{c}(b_i)$ cannot be represented as a linear combination of the other $N-1$ column vectors of $\mathbf H$. Contrary to this assumption, the proposed proof aims to show that there exists a (nonzero) $N$-dimensional vector $\alpha$ such that it is orthogonal to the vector $\mathbf{c}(b_i)-\mathbf{c}( a)$. In particular, the proposed proof aims to show that
\begin{equation}
\label{eq-1}
\langle\alpha,\mathbf{c}(b_i)-\mathbf{c}(a)\rangle=0.
\end{equation}
It is clear that the proposed formulation is not equivalent to the required one. However, even the proof for the proposed reformulation is incorrect. At one point of the proof, by reductio ad absurdum, authors of \cite{Huang06} tried to obtain a contradiction from the following equation, which is equivalent to \eqref{eq-1}:
\begin{equation}
\label{eq2}
g(\mathbf{w}_i \cdot\mathbf{x}_N+b_i)=-\sum_{p=1}^{N-1}\gamma_p g(\mathbf{w}_i \cdot\mathbf{x}_p+b_i)+z/\alpha_N,
\end{equation}
where $\alpha_N\neq 0$, $z=\langle\alpha,\mathbf{c}(a)\rangle$, and $\gamma_p=\frac{\alpha_p}{\alpha_N}$, $p=1,\ldots,N-1$.
With this purpose in mind, they sequentially apply to both parts of (\ref{eq2}) the operation of differentiation ``with respect to $b_i$'', and obtain a  set of countably many linear equations with $N-1$ unknown coefficients as follows:
\begin{equation}
\label{eq+1}
g^{(l)}(\mathbf{w}_i \cdot\mathbf{x}_N+b_i)=-\sum_{p=1}^{N-1}\gamma_p g^{(l)}(\mathbf{w}_i \cdot\mathbf{x}_p+b_i),\qquad l=1,\ldots, N,\ldots,
\end{equation}
where $g^{(l)}$ is the $l$-th derivative of function $g$. 
Literally they assert:
 \begin{quote}
     \em However, there are only $N-1$ free coefficients $\gamma_1,\gamma_2\dots\gamma_{N-1}$ for the derived more than $N-1$ linear equations, this is contradictory. 
 \end{quote}

 Even assuming that something is missing in the text about the author's reasoning, the falsity of this reasoning easily follows from two facts:
  \begin{itemize}
  \item Since the original equation \eqref{eq2} is valid just for a finite set of specific values of the argument of the activation function $g$, it may not hold for the same values of the argument after replacing the activation function $g$ with any of its derivatives. Example \ref{examContDer} below illustrates this point.

\begin{example}\label{examContDer}
For a numerical example, we choose the function $g(x)=\sin{x}$ on the interval $[0,2\pi]$ and consider the Equation (\ref{eq2}) with the following parameters: $N=3$, $a=0, b_i=\pi$, $w_i=\pi$, $w_0=-\pi$, $x_1=1/4$, $x_1=1/2$, $x_1=3/4$, so that $\mathbf{c}(b_i)=(g(b_i+w_ix_1), g(b_i+w_ix_2),g(b_i+w_ix_3))$,
$\mathbf{c}(a)=(g(a+w_0x_1), g(a+w_0x_2),g(a+w_0x_3))$. After substitution and computation, we obtain
that $\mathbf{c}(b_i)=(-\sqrt{2}/2, -1, -\sqrt{2}/2)$, $\mathbf{c}(a)=(-\sqrt{2}/2, -1, -\sqrt{2}/2)$, i.e. $\mathbf{c}(b_i)=\mathbf{c}(a)$. Therefore, the equation (\ref{eq-1}) has infinitely many solutions, and we can take $\alpha_1=1,\alpha_1=2,\alpha_1=3$ as one of the possible solutions. Then the equivalent (\ref{eq2}) of the equation (\ref{eq-1}) has the form:
\begin{equation}
\label{eq+2}
g(b_i+w_ix_3)=(-1/3)g(b_i+w_ix_1)-(2/3)g(b_i+w_ix_2)+(1/3) (\alpha,\mathbf{c}(a)),
\end{equation}
which is obviously correct.
However, if assuming that $b_i$ denotes the independent variable, we formally apply differentiation to both sides of the equation (\ref{eq+2}), we will not obtain the correct equality. Indeed,
\[ 
g'(b_i+w_ix_3)=(-1/3)g'(b_i+w_ix_1)-(2/3)g'(b_i+w_ix_2),
\]
does not hold for the derivative $g'(x)=\cos{x}$ because
\[ 
\cos{(\pi+3/4\pi)}\neq (-1/3)\cos{(\pi+1/4\pi)}-(2/3)\cos{(\pi+1/2\pi)}.
\]   
\end{example}

     \item There is no contradiction in the statement that there are compatible systems of more than $N$ (or even infinitely many)   linear equations with $N-1$ unknowns. As a trivial example, consider a functional tautology that remains true after replacing all functions with their derivatives.   Example \ref{Ex2} below illustrates this point.
    \begin{example}
\label{Ex2} 
Let us consider the infinitely differentiable real functions $f(x)=\cos(2x)$, $g(x)= \cos^2(x)$ and $h(x)=\sin^2(x)$. Then,  the equality:  
         $$
         f(x)=\gamma_1g(x)+\gamma_2 h(x)
         $$
         is tautologically true for all $x\in\R$ if $\gamma_1=1$ and $\gamma_2=-1$. If we compute the $k$-th derivative in each side of the equality, we can easily check that the equality
         $$
         f^{(k)}(x)=\gamma_1g^{(k)}(x)+\gamma_2 h^{(k)}(x)
         $$
         also holds for $x\in\R$ if $\gamma_1=1$ and $\gamma_2=-1$. Consequently, fixed $x\in\R$, the infinitely  many generated lineal equations form a system 
         $$
        \left\{ f^{(k)}(x)=\gamma_1g^{(k)}(x)+\gamma_2 h^{(k)}(x)\right\}_{k\in\mathbb{N}}
         $$
        which  is compatible, since $\gamma_1=1$ and $\gamma_2=-1$ is always a solution.
     \end{example}
    
 \end{itemize}
\end{enumerate}
Thus, the proposed proof in \cite{Huang06} does not support the assertion given by Theorem~2.1.

\subsubsection*{Further counterexamples to Theorem 2.1}\label{Th21Cexamples}

In this section, we aim to show that the statement of Theorem 2.1 in \cite{Huang06} is fundamentally false by providing several counterexamples. All of them show special cases that are consistent with the assumptions of Theorem 2.1 but contradict its conclusion.

\begin{itemize}
    \item If the activation function $g$ is constant, then all elements of the matrix
$\mathbf H =(h_{ji})$, where $h_{ji}=g(\mathbf{w}_i\cdot \mathbf{x}_j+b_i)$, are equal, so $\mathbf H$ is singular (has no inverse). This case is not excluded under the conditions of Theorem 2.1, since the constant function is infinitely differentiable.

\item If among $N$ different samples $(\mathbf x_j, \mathbf t_j)$ there are two that differ only in the second components, i.e. for some $j_1\neq j_2$ we have $\mathbf x_{j_1}=\mathbf x_{j_2}$ and $\mathbf t_{j_1}\neq \mathbf t_{j_2}$, then the $j_1$-th and $j_2$-th rows of the matrix $\mathbf H =(h_{ji})$ coincide, which means $\mathbf H$ is singular.
\item If among $N$ pairs $(\bf w_i,b_i)$ of weight vectors $\bf w_i$ and bias values $b_i$, selected according to a continuous probability distribution, there are two that coincide, i.e. $\bf w_{i_1}=\bf w_{i_2}$ and $b_{i_1}=b_{i_2}$ for some $i_1\neq i_2$, then the $i_1$-th and $i_2$-th columns of the matrix $\mathbf H =(h_{ji})$ coincide, which means $\mathbf H$ is singular.
 \item If the activation mapping $g\colon \R\to \R$ is periodic, that is, there is $p\in\R$ such that $g(x+p\cdot k)=g(x)$ for all $x\in\mathbb{R}$ and $k\in\mathbb{N}$, then among randomly selected pairs $(\bf w_i,b_i)$ of weight vectors $\bf w_i$ and bias values $b_i$, configuration $\bf w_{i_1}=\bf w_{i_2}$ and $b_{i_1}=b_{i_2}+p$ can arise for some $i_1\neq i_2$. In this case, the $i_1$th and $i_2$th columns of the matrix $\mathbf H =(h_{ji})$ coincide, which means $\mathbf H$ is singular.
\end{itemize}

\subsection{Refutation of the proof of  \cite[Theorem 2.2]{Huang06}}

We will show here that  Theorem 2.2, does not have a correct proof either.  Its general assumptions, up to the choice of the parameter $\tilde N$, are the same as in Theorem 2.1. Below we recall (literally) the statement:

\noindent
 \textbf{Theorem 2.2}   {\it Given any small positive value $\varepsilon >0$  and activation function $g\colon\R\to \R$ which is infinitely differentiable in any interval, there exists $\tilde N\leq N$ such that for $N$ arbitrary distinct samples $\{(\mathbf{x}_i,\mathbf{t}_i)\mid i=1,\ldots, N\}$, where $\mathbf{x}_i\in \R^n$, and $\mathbf{t}_i\in \R^m$, for any $\mathbf{w}_i$ and $b_i$ randomly chosen from any intervals of $\R^n$ and $\R$, respectively, according to any continuous probability distribution, then with probability one, $\|\mathbf H\beta-\mathbf T\|<\varepsilon$}.
 \par\smallskip

To begin with, the statement suffers from the same shortcomings as those noted in relation to Theorem 2.1. A first remark refers to the fact that the proofs of both assertions of Theorems 2.1 and 2.2 do not use any probability distributions. A second remark is that the actual proof of Theorem 2.2 is not given, but only noted that  ``{\it The validity of the theorem is obvious, otherwise, one
could simply choose $\tilde N=N$.''
}

Regardless the fact that Theorem 2.1 has not been correctly proven, to refute this reasoning, we note that   there is no apparent relationship between the norm of the difference $\|\mathbf H\beta-\mathbf T\|$ and the number $\tilde N$ of neurons in the hidden layer. Moreover, with each new value of $\tilde N$, all neural network parameters and hence the matrix $\mathbf H$ are updated due to the set of random choice of weights and biases. Therefore, there is neither regularity nor stability in the elements of the matrix $\mathbf H$, so the assertion that a $\tilde N$ guarantees the desired inequality is not confirmed.

\subsection{\label{S3}On the minimum norm least-squares solution of SLFNs}

After analyzing in  previous sections the flaws in the proofs of Theorems~2.1 and~2.2 in~\cite{Huang06},  we focus here on the proposed minimum norm least-squares solution of SLFNs \cite[Section 3.2]{Huang06}, where we can find the following statement
 ``{\it For fixed input weights $\mathbf{w}_i$ and hidden layer biases $b_i$,   to train an SLFN is simply equivalent to finding a least squares solution $\hat \beta$ of the linear system (\ref{SLFNmodel2}), i.e. $\mathbf H\beta=\mathbf T$}.'' 
This least squares solution is proposed to be calculated using the generalized inverse Moore-Penrose (pseudo-inverse) matrix so that
\begin{equation}
\label{eq3bis}
\hat \beta= {\mathbf H}^\dagger \mathbf T.
\end{equation}

It is well-known \cite{Penrose56} that, in general, a pseudo-inverse matrix ${\mathbf H}^\dagger$ that solves (\ref{SLFNmodel2}) may not exist, and if it exists, then it may not be unique.

The matrix $\hat \beta$ in (\ref{eq3bis}) solves the problem only in the ``least squares'' sense, that is, we have 
\[ 
  \|\mathbf H\beta-\mathbf T\|_F\geq \|\mathbf H\hat \beta-\mathbf T\|_F,
\]
for all $\beta\in\R^{\tilde N\times m}$, where $\|\cdot\|_F$ denotes the Fr\"obenius norm. As a result, this means that the norm $\|\mathbf H\hat \beta-\mathbf T\|_F$ is a lower bound of the error of the proposed ELM methodology.

In consequence, based on a random and then a fixed choice of the influencing parameters $\mathbf{w}_i$, $b_i$, $i=1,\ldots, \tilde N$, and $\mathbf{x}_j,\, j=1,\ldots, N$, of the matrix $\mathbf H$, we can neither guarantee that the system (\ref{SLFNmodel2}) has a solution, nor that the distance between the matrix $\mathbf T$ and the set of matrices $\{\mathbf H\beta\mid \beta\in \R ^{\tilde N\times m} \}$ is smaller than any prescribed $\varepsilon$.

\section{Critical Notes on the Performance Evaluation}
\label{S4}


This section focuses on the part of the paper \cite{Huang06} that discusses the comparison of some test cases computed by the proposed ELM learning algorithm  with other learning strategies, namely traditional feed-forward network learning algorithms such as back-propagation (BP). 

The first benchmark problem was the approximation of the “sinc function”\footnote{The unnormalized sinc function is defined for $x\neq 0$ by $\mathrm{sinc}\, x={\frac {\sin x}{x}}.$} with noise, where the training and testing sets contained 5000 data pairs $(x_i, y_i)$ and $x_i$, uniformly randomly distributed over the interval $(-10,10)$. For both learning algorithms, ELM and BP, the number of hidden nodes was 20.
The authors of \cite{Huang06} stated that after 50 tests, the average results and standard
deviations showed that accuracy and deviations were similar, while ELM performed 170 times faster than the BP algorithm. 

Commenting on these results, we will say that they do not confirm Theorems~2.1 and 2.2 and, moreover, they are not at all related to the theoretical part proposed in \cite{Huang06}, section 2. In detail:
\begin{itemize}
\item Theorem~2.1 states that when the number of training pairs $N$ coincides with the number of neurons in a hidden layer $\tilde N$, then the ELM solves the interpolation problem, or more precisely, there exists $\beta$ such that $\|\mathbf H\beta-\mathbf T\|=0$.
\item Theorem~2.2 states that the difference $\|\mathbf H\beta-\mathbf T\|$ can be made arbitrarily small for some existing $\tilde N\leq N$.
\end{itemize}  

Both theorems concern the ability of ELM to model input-output relationships defined by $N$ samples as accurately as possible. However, the obtained results of approximating the “sinc function with noise” are far from accurate with respect to the input-output pairs used for training. The latter was specifically chosen by the authors of \cite{Huang06}, who defined it as follows (see last paragraph on page 5): ``{\it large uniform noise distributed in $[-0.2,0.2]$ has been added to all the training samples.}''

After tuning the ELM parameters, testing was performed on noise-free data (not used in training), and an accuracy estimate for the ``noise-free sinc function'' was shown. The question of why training pairs were not taken into account and success was indicated by results other than expected was not raised.

Last but not least, in the test cases, the actual ELM parameters were not calculated ``in one run'' as proposed by the ELM algorithm. Instead,
``{\it 50 trials have been conducted for all the algorithms and the average results [\dots] are shown in Table 1.}''~\cite[pg. 494]{Huang06}.%
This means that when running tests, the authors contradict their own ELM algorithm, which is based on three deterministic steps and does not involve any trials.

\section{Towards a correct version of~\cite[Theorem 2.1]{Huang06}}\label{SecTeoNew}

In this section, we present a result that, by slightly changing the conditions of Theorem 2.1, proves its validity in probability spaces.
In doing so, we take a first step towards a theoretical support for ELM.
Two facts should be taken into account: on the one hand, the ELM method works for a certain number of applications (this fact is known from many published papers), and on the other hand, this method is based on an incorrectly proven statement~\cite[Theorem~2.1]{Huang06}. Hence, our attempt to provide a correct proof of~\cite[Theorem~2.1]{Huang06} implicitly justifies existing applications.

\subsection{On the random selection of weights and biases}

Below we present a theoretical result that refines the formulation of~\cite[Theorems 2.1]{Huang06} by imposing conditions on the dataset considered for training and the deterministic parameters (activation function and number of neurons) of the neural network. For the sake of presentation and understanding,  we have removed the reference to probability theory from the proposed statement, since it is not considered in the original proof of~\cite[Theorems 2.1]{Huang06}; such a link is left for the subsequent  Section~\ref{probELM}. Note also that in the formulation of Theorem 2.1, the dataset consists of input-output vectors, denoted by $\bf x_i$ and $\bf t_i$, and the non-deterministic parameters, i.e. vectors of weights and bias values, are denoted by 
$\bf w_i$ and $b_i$. 

\begin{theorem}\label{teoNew}
For a given standard SLFN with $N$ distinct input-output pairs $(\bf{x_i},t_i)$ 
(where $\mathbf x_i\in\R^n$ and $\mathbf t_i\in\R^m$) and $N$ hidden nodes, where
\begin{itemize}
    \item the activation function $g\colon \R\to \R $ obeys the properties: $g\in\mathcal{C}^1$ ( i.e.  $g$ is differentiable and $g'$ is continuous) and the set of critical points $\{x\in\mathbb{R}\mid g'(x)= 0\} $ is countable,
    \item the input vectors $\bf{ x_i}$ are lineally independent,\footnote{Implicitly, we require that $n\geq N$.}
\end{itemize}
it is true that the interior\footnote{Interior with respect to the standard topology of $\mathbb{R}^{N(n+1)}$, i.e., $\bf x\in \mathbb{R}^{N(n+1)}$ is in the interior of a set $S\subseteq \mathbb{R}^{N(n+1)}$ if there exist a ball $B({\bf x},\varepsilon)$ with center $\bf x$ such that $B({\bf x},\varepsilon)\subseteq S$.} of the set 
\[ 
W=\{(\mathbf{w}_1,b_1,\dots,\mathbf{w}_N,b_N)\in \R^{N(n+1)}\mid  \text{$\mathbf H$ of the SLFN is not invertible}\},
\] 
 is empty.
\end{theorem}

\begin{proof}  Let us as proceed by induction over $N$. If $N=1$, then  $\mathbf H=g(\mathbf w_1\mathbf x_1+b_1)$ and thus, $\mathbf H$ is not invertible if and only if $g(\mathbf w_1\mathbf x_1+b_1)=0$. 
On the one hand, note that from ``$\{ x\in\mathbb{R}\mid g'(x)= 0\} $ is countable''  and Rolle's Theorem\footnote{Formally speaking, here we use the following extended version of Rolle's Theorem: given $f\colon \mathbb{R}\to\mathbb{R}$ derivable on $\mathbb{R}$, if $\lim_{x\to-\infty}f(x)=\lim_{x\to\infty}f(x)$ then there exists $c\in\mathbb{R}$ such that $f'(c)=0$.} we can assert that the set   $\{ x\in\mathbb{R}\mid g(x)= 0\} $ is countable as well. This assertion can be proved by dividing the set of real numbers in a numerable set of disjoint intervals which bounds are points in $\{ x\in\mathbb{R}\mid g'(x)= 0\} $. Then, by Rolle's Theorem, there is at most one root of $g$ in each of those intervals. Consequently, $\{ x\in\mathbb{R}\mid g(x)= 0\} $ is countable as well. 
On the other hand, for each $c$ such that $g(c)= 0$, the set $({\bf w_1},b_1)\in\mathbb{R}^{n+1}$ such that $\mathbf w_1\mathbf x_1+b_1=c$ is a  hyperplane of dimension $n$ in $\mathbb{R}^{n+1}$. As a result, the set of weights and biases $({\bf w_1},b_1)\in\mathbb{R}^{n+1}$ such that $\mathbf H$ is not invertible, is a countable union of hyperplanes in $R^{n+1}$, whose interior is empty in $R^{n+1}$.

Let $N>1$ and let us assume that the result holds for all $\overline N< N$. Let us proceed by reductio ad absurdum and let us assume that the interior $\mathring W$ of the set $W$ of weights and bias vectors  such that
$\mathbf H$ is not invertible is nonempty.   Note that, for all $(\mathbf{w}_1,b_1,\dots, \mathbf{w}_N,b_N) \in \mathring W$ we have that the determinant of $\mathbf H$ is zero.  Let us consider the mapping that provides the result of the determinant of $\mathbf H$ for each combination of weights and biases, formally, $F\colon\mathbb{R}^{N(n+1)}\to\mathbb{R}$, given by $F(\mathbf{w}_1,b_1,\dots, \mathbf{w}_N,b_N)=\lvert \mathbf H\rvert$. On the one hand, we have that $F$ is differentiable because $g\in\mathcal{C}^1$; on the other hand, 
we also have $F(\mathbf{w}_1,b_1,\dots, \mathbf{w}_N,b_N)=0$ for all  $(\mathbf{w}_1,b_1,\dots, \mathbf{w}_N,b_N) \in \mathring W$, namely, $F$ is constant in the open set $\mathring W$. As a result we obtain that, for all $(\mathbf{w}_1,b_1,\dots, \mathbf{w}_N,b_N) \in\mathring W$
\begin{equation}\label{partials=0}
 \dfrac{\partial F}{\partial w_{ij}}(\mathbf{w}_1,b_1,\dots, \mathbf{w}_N,b_N)=0\,.
\end{equation}

In order to obtain an analytic formula of the partial derivatives with respect to $w_{1k}$ associated to the weight vector $\mathbf{w}_1$, let us consider a Laplace expansion (also called cofactor expansion) along the $1$st column  which results on the following expression of $F$
$$
F(\mathbf{w}_1,b_1,\dots, \mathbf{w}_N,b_N)=\lvert \mathbf H\rvert= \sum_{j=1}^{N} (-1)^{j+1}h_{j1}\cdot M_{j1}=
    \sum_{j=1}^{N} (-1)^{j+1}g(\mathbf{w}_1\cdot \mathbf{x}_j+b_1)\cdot M_{j1}
$$
which leads to the following formulae:
$$
\dfrac{\partial F}{\partial w_{1k}}(\mathbf{w}_1,b_1,\dots, \mathbf{w}_N,b_N)=
 \sum_{j=1}^{N}  x_{jk} (-1)^{j+1}g'(\mathbf{w}_1\cdot \mathbf{x}_j+b_1)\cdot M_{j1}
 $$
Recalling \eqref{partials=0}, we have  $
\dfrac{\partial F}{\partial w_{1j}}(\mathbf{w}_1,b_1,\dots, \mathbf{w}_N,b_N)=0
$ 
for all $(\mathbf{w}_1,b_1,\dots, \mathbf{w}_N,b_N) \in\mathring W$, and we obtain the following  $n$ equations  that hold for all $(\mathbf{w}_1,b_1,\dots, \mathbf{w}_N,b_N) \in\mathring W$:  
\begin{equation}\label{sistemequth1}
    \sum_{j=1}^{N}  x_{jk} (-1)^{j+1}g'(\mathbf{w}_1\cdot \mathbf{x}_j+b_1)\cdot M_{j1}=0    \qquad \text{for } k\in\{1,\dots,n\}.
\end{equation}
If we apply the change of variables $y_j=(-1)^{j+1}g'(\mathbf{w}_1\cdot \mathbf{x}_j+b_1)M_{j1}$, we obtain a homogeneous linear system of $n$ equations and $N$ unknowns with associated matrix of coefficients $\mathbf A=(x_{jk})$. Since the vectors $\bf{x}_j$ are lineally independent, the rank of $\mathbf A$ is $N$ 
and then, the only solution of the system is the trivial one, $y_j=0$, that is $(-1)^{j+1}g'(\mathbf{w}_1\cdot \mathbf{x}_j+b_1)M_{j1}=0$, for all $j\in\{1,\dots,N\}$ and $(\mathbf{w}_1,b_1,\dots, \mathbf{w}_N,b_N) \in\mathring W$. Hence, for each $j$, we have two possible cases:
\begin{enumerate}
\item $g'(\mathbf{w}_1\cdot \mathbf{x}_j+b_1)=0$ 
\item  $M_{j1}=0$
\end{enumerate}

We will prove now that the latter case always holds. Since   case (1) above takes into consideration just   the weight $\mathbf{w}_1$ and the bias $b_1$, we will only focus on those variables. Let us denote by $\pi(\mathring W)$ the projection of $\mathring W$ on its first $n+1$ components, fix $j\in{1,\dots,N}$, and consider the following set:
\begin{align*}
    B&=\{(\mathbf{w}_1,b_1)\in\pi(\mathring W)\subseteq \mathbb{R}^{n+1}\mid g'(\mathbf{w}_1\cdot \mathbf{x}_j+b_1)\neq 0\}
\end{align*}
Note that if the projection of a vector $(\mathbf{w}_1,b_1,\dots, \mathbf{w}_N,b_N) \in \mathring W$ on its first $n+1$ components belongs to $B$ then $M_{j1}=0$. 
Let us see now that if the projection of a vector $(\mathbf{w}_1,b_1,\dots, \mathbf{w}_N,b_N) \in \mathring W$ on its first $n+1$ components belongs to $\pi(\mathring W)\smallsetminus B$, then $M_{j1}=0$ as well. 

By following a similar reasoning than the one done by the case $N=1$, we obtain that  $\pi(\mathring W)\smallsetminus B$ is contained in a countable union of hyperplanes in $\mathbb{R}^{n+1}$. As a result, the set of accumulation points of $B$  coincides with the set of accumulation points of  $\pi(\mathring W)$. Consequently,\footnote{We are using the following standard result: Given $X\subseteq \mathbb{R}^k$, the projection of an accumulation point of $X$ is an accumulation point of any projection of $X$.} the projection of an accumulation point of $\mathring W$ is an accumulation point of $B$.
Since the determinant   $M_{j1}\colon \mathbb{R}^{N(n+1)} \to \mathbb{R}$ defines a continuous function, then\footnote{We are using the following standard result:
\emph{Let $f\colon\mathbb{R}^k\to\mathbb{R}$ be a continuous function such that $f(\bf x)=0$ for all $\bf x\in D\subseteq\mathbb{R}^k$, then $f(\bf x)=0$ for all $\bf x\in\overline D$, where $\overline {\mathbf D}$ is the set of accumulation points of $\mathbf D$.} } $M_{j1}=0$.

Note now that each cofactor $M_{j1}$ depends just on the weights $(\mathbf{w}_2,b_2,\dots, \mathbf{w}_N,b_N)$, that is, the weights of the projection of $W$ into the  last $(N-1)(n+1)$  components, and let us write $\pi_2 ( W)$ to denote the projection on these last  components.  Note also that each cofactor  $M_{j1}$ corresponds to a \emph{hidden layer output matrix} of a SLFN with $N-1$ hidden nodes. Therefore, we can apply the induction hypothesis, for instance with respect  to $M_{11}$ and ensure that the interior of the  set of weights and biases such that $M_{11}=0$ is empty in $\mathbb{R}^{(N-1)(n+1)}$.  Since $\pi_2(W)$ is contained in the set of weights and biases such that $M_{11}=0$ , the interior of $\pi_2(W)$ is empty as well.

Summarizing, we have
\begin{itemize}
\item $W\subseteq \mathbb{R}^{n+1}\times \pi_2(W)$,
\item the interior of $\mathbb{R}^{n+1}\times \pi_2(W)$ is empty;
\end{itemize}
which implies that the interior of $W$ is empty, which contradicts the hypothesis of reductio ad absurdum.
\end{proof}

\subsection{Relationship between~\cite[Theorem 2.1]{Huang06} and Theorem~\ref{teoNew}}\label{probELM}

As stated at the beginning of this section, the aim of Theorem~\ref{teoNew} is to establish a sound theoretical background to the ELM technique. Therefore, it is important to indicate their similarities and differences.
Let us recall that the consequent of~\cite[Theorem 2.1]{Huang06} is
\begin{quote}
\textit{``the probability of choosing  weights and biases such that $\mathbf H$ is not invertible is $0$''}
\end{quote}
whereas the consequent of our Theorem~\ref{teoNew} is      
    \begin{quote}    
     \textit{`` the interior of the set of $\mathbf{w}_i$ and $b_i$, of $\R^n$ and $\R$, respectively, such that $\mathbf H$ of the SLFN is not invertible, is empty in $\R^{N(n+1)}$''.} 
       \end{quote}
       
A very common approach~\cite{Billingsley95} to assign probabilities to the choice of vectors in $\mathbb{R}^n$  is by means of a measurable function $f$, and the identification of the probability of a set $S$ by $P(S)=\int_{S} f(x)dx$. 
With the previous identification, we need only consider general results from measure theory and probability theory, which we assume the reader is familiar with.
 
\begin{corollary}\label{CorProb}
 Let $W$ be the set of weights and biases $\mathbf{w}_i$ and $b_i$, of $\R^n$ and $\R$, respectively, such that the matrix $\mathbf H$ of the SLFN is not invertible. 
 Let $(\mathbb{R}^{N(n+1)},\mathcal{F}, P)$ be a probability space where the probability measure $P$ is defined in terms of a measurable function $f\colon \mathbb{R}^{N(n+1)}\to \mathbb{R}$ as $P(S)=\int_{S} f(x)dx$ for all $S\in\mathcal{F}$. Under the hypotheses of Theorem~\ref{teoNew}, if $W\in\mathcal{F}$ and $W$ is Jordan measurable, then $P(W)=0$, i.e.,  choosing a set of weights and biases $\mathbf{w}_i$ and $b_i$ such that $\mathbf H$ is not invertible has probability zero. 
\end{corollary}
 
 \begin{proof}
First, since $W$ is Jordan measurable, its boundary $\partial W$   has measure $0$; consequently, $\int_{\partial W} f(x)dx=0$.
Second, since   $\mathring W$ is empty (by Theorem~\ref{teoNew}), necessarily $W\subseteq \partial W$. By monotonicity of integrals $\int_{W} f(x)dx\leq \int_{\partial W} f(x)dx$.
Joining both facts, we have that
$$
P(W)=\int_{W} f(x)dx\leq \int_{\partial W} f(x)dx=0
$$
\end{proof}

Notice that we have not proved,  in general, that  choosing a set of weights and biases $\mathbf{w}_i$ and $b_i$ such that $\mathbf H$ is not invertible has probability zero. Corollary~\ref{CorProb} is given in the framework of continuous probability spaces, whereas computers work on discrete probability spaces. Moreover, we have imposed some important conditions in Theorem~\ref{teoNew} and also in Corollary~\ref{CorProb} that have to be taken into account; note that in the latter we have assumed additionally that the set of those weights and biases is  Jordan measurable. 

Finally, let us note that we have shown that our new results in this section provide an interesting sound alternative statement  to that of the original~\cite[Theorem 2.1]{Huang06}. In particular, thanks to Theorem~\ref{teoNew} and Corollary~\ref{CorProb} we can state that under their conditions, a random choice of weights $\mathbf{w}_i$ and biases $b_i$ leads to the very unlikely situation where the matrix $\mathbf H$ is not invertible. 

\section{ELM that exactly learns dataset $S$ exists}
\label{S7}

In this section, we show that the dataset $S$ proposed in Example~\ref{Ex3} (Section~\ref{S5}) as a counterexample to the ELM learning algorithm can be learned by ELM with hyperparameters different from those specified in Theorem~2.1 from \cite{Huang06}.

The ELM configuration we found, which successfully reproduces all 400 samples of the $S$ dataset from Example~\ref{Ex3}, uses ReLU (rectified linear unit) activation function and 8000 hidden neurons instead of the 400 stated by Theorem~2.1.
Thus, the found ELM configuration refuted two conditions in the formulation of Theorem 2.1 from \cite{Huang06}, namely: the activation function must be infinitely differentiable and the number of hidden neurons must be equal to the number of training samples.

Below in Figure~\ref{Fig1} we show two graphs illustrating the results of several runs with randomly selected weights and biases of the ELM learning algorithm with the ReLU activation function trained on a dataset $S$ consisting of $N=400$ training pairs. The number of hidden nodes is $\tilde N= 400$ (left) and $\tilde N= 8000$ (right). We again see that this newly proposed ELM configuration with $N=\tilde N$ cannot reproduce all $N=400$ samples of dataset $S$.

\begin{figure}
\includegraphics[height=5.75cm,width=7cm]{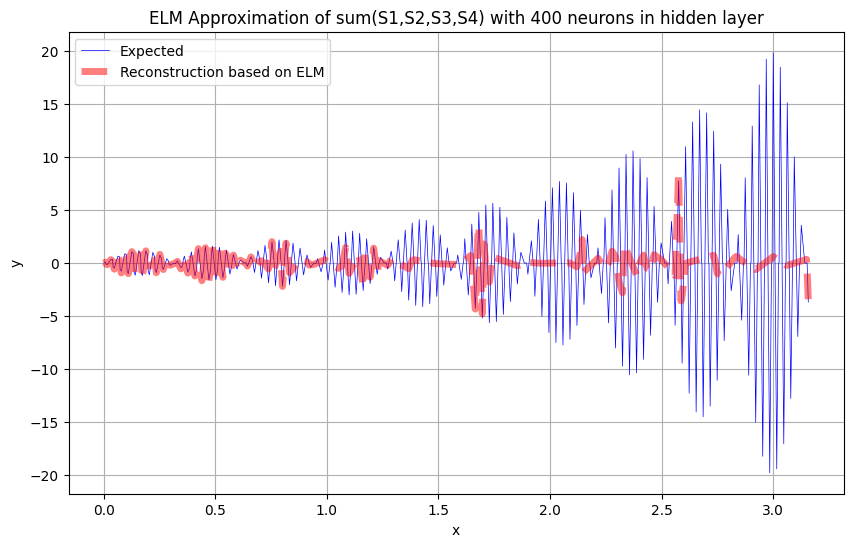}
\includegraphics[height=5.75cm,width=9cm]{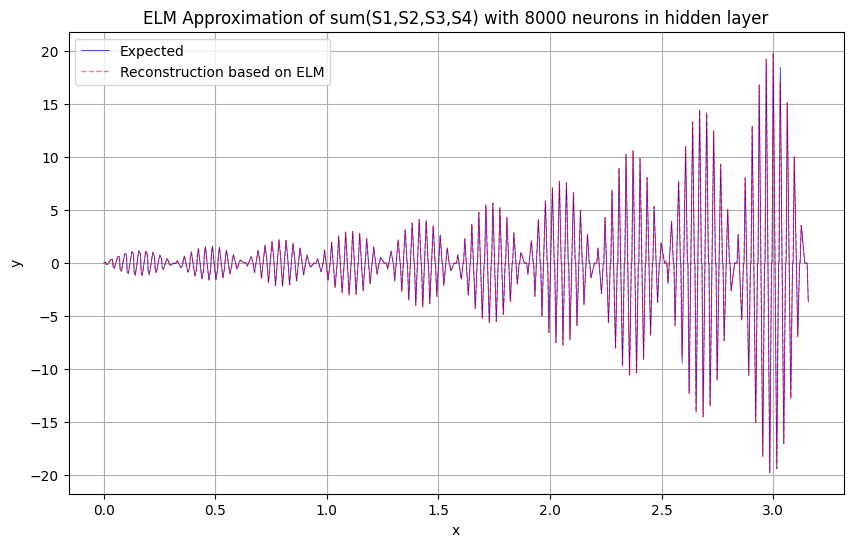}
\caption{Two plots illustrating the results of several runs with randomly selected weights and biases of the ELM learning algorithm with the ReLU activation function trained on a dataset $S$ consisting of $N=400$ training pairs. In the case of $\tilde N= 400$ hidden nodes (left), this ELM cannot reproduce all $N=400$ samples of dataset $S$. In the $\tilde N= 8000$ case (right), all $N=400$ samples of the $S$ dataset were successfully reproduced.}\label{Fig1}
\end{figure}

\section{Conclusion}\label{conclusion}

We have presented a reasonable critique of the principles of the so-called extreme learning machine methodology proposed in ~\cite{Huang06}. We began the analysis by introducing a counterexample to the methodology. We have shown that when using a standard SLFN architecture with random assignment of input weights and biases along with a random selection of an infinitely differentiable activation function, accurate reproduction of all training samples cannot be guaranteed if the numbers of hidden nodes and training samples are the same. This fact means that the theoretical framework presented in~\cite{Huang06} must have shortcomings.
In this regard, we analyzed and refuted the proofs of \cite[Theorems~2.1 and Theorems~2.2]{Huang06}. In particular, we pointed out several significant errors in both proofs and provided several families of counterexamples to these results. In our opinion, the biggest drawback is the identification of impossible events with events with probability $0$ in continuous probability spaces. 

Although the fundamentals of ELM are not sufficiently correct, from a practical point of view they have proven to be effective in solving many problems and therefore should not be completely rejected. In this work, we provided a new theoretical result that partially supports this method under certain conditions while preserving probabilistic uncertainty.
Finally, we showed that the dataset proposed as a counterexample to the ELM learning algorithm can be learned by ELM with hyperparameters different from those specified in
Theorem~2.1 from \cite{Huang06}.

Our further work continues in this direction. In addition to the mentioned theoretical result, which can be further improved, we aim to develop other results that will allow some modification of the ELM technique, for example, by changing the random selection of weights and biases.

\section*{Acknowledgement}
The article has been supported by the Polish National Agency for Academic Exchange Strategic Partnership Programme under Grant No. BPI/PST/2021/1/00031.

M. Ojeda-Aciego and N. Madrid are partially supported by the Ministry of Science and Innovation  (MCIN), the State Agency of Research (AEI) and the European Regional Development Fund (FEDER) through the project  VALID,\newline PID2022-140630NB-I00/AEI/10.13039/501100011033/ FEDER, UE.

\section*{References}
\bibliographystyle{abbrv}
\bibliography{elm,scopus-selection-23-24,scopus-neurocomputing}

\end{document}